\documentclass[final]{l4dc2024}

\usepackage{times}
\usepackage{amsmath,amsfonts,bm}
\usepackage{hyperref}
\usepackage{url}
\usepackage{amsmath}
\usepackage{amssymb}
\usepackage{mathtools}
\usepackage{bm}
\usepackage{amssymb}
\usepackage{times}
\usepackage{wrapfig}
\usepackage{mathabx}
\usepackage{algorithmic}
\usepackage{algorithm}
\usepackage{algorithmic}
\usepackage{algorithm}
\usepackage{wrapfig}
\usepackage{lipsum}

\newcommand{\state}{{\bm{s}}}
\newcommand{\action}{{\bm{a}}}

\newcommand{\obs}{{\bm{x}}}
\newcommand{\rew}{{\bm{r}}}
\newcommand{\policy}{{\pi}}

\newcommand{\dynamics}{{\mathcal{T}}}

\newcommand{\mdp}{\mathcal{M}}
\newcommand{\mdphat}{{\widehat{\mathcal{M}}}}

\title[Imitation learning with conservative world models]{Efficient Imitation Learning with Conservative World Models}
\author{%
 \Name{Victor Kolev$^*$} \Email{vkolev@cs.stanford.edu}\\
 \addr Stanford University
 \AND
 \Name{Rafael Rafailov$^*$} \Email{rafailov@cs.stanford.edu}\\
 \addr Stanford University
 \AND
 \Name{Kyle Hatch} \Email{kyle.hatch@tri.global}\\
 \addr Toyota Research Institute
 \AND
 \Name{Jiajun Wu} \Email{jiajunwu@cs.stanford.edu}\\
 \addr Stanford University
 \AND
 \Name{Chelsea Finn} \Email{cbfinn@cs.stanford.edu}\\
 \addr Stanford University
}

\begin{document}

\maketitle
\def\thefootnote{*}\footnotetext{Equal contribution.}%
\def\thefootnote{\arabic{footnote}}

\begin{abstract}%

We tackle the problem of policy learning from expert demonstrations without a reward function. A central challenge in this space is that these policies fail upon deployment due to issues of distributional shift, environment stochasticity, or compounding errors. Adversarial imitation learning alleviates this issue but requires additional on-policy training samples for stability, which presents a challenge in realistic domains due to inefficient learning and high sample complexity. One approach to this issue is to learn a world model of the environment, and use synthetic data for policy training. While successful in prior works, we argue that this is sub-optimal due to additional distribution shifts between the learned model and the real environment. Instead, we re-frame imitation learning as a fine-tuning problem, rather than a pure reinforcement learning one. Drawing theoretical connections to offline RL and fine-tuning algorithms, we argue that standard online world model algorithms are not well suited to the imitation learning problem. We derive a principled conservative optimization bound and demonstrate empirically that it leads to improved performance on two very challenging manipulation environments from high-dimensional raw pixel observations. We set a new state-of-the-art performance on the Franka Kitchen environment from images, requiring only 10 demos on no reward labels, as well as solving a complex dexterity manipulation task. 

\end{abstract}


\section{Introduction}

Learning by demonstration is a natural way for agents to learn complex behaviors. A small set of demonstrations is often easy to obtain via tele-operation from a human expert, while dense reward functions are non-trivial to design, and sparse rewards are challenging for reinforcement learning methods. Demonstrations further alleviate the need for exploration, as they forego both the search problem and the exploration-exploitation trade-off, and instead reduce the task to distribution matching. The enduring approach to learning from demonstrations is Behavior Cloning (BC) \citep{pomerleau1988alvinn}, which, however, suffers from compounding policy errors, instability due to environment stochasticity, and covariate shift \citep{Dagger2011Ross}. Ways to alleviate these shortcomings involve more expert data -- wider expert coverage \citep{Feedback2021Spencer} or directly querying the expert \citep{Dagger2011Ross}. 


Alternative approaches to BC are inverse RL (IRL) \citep{finn2016guided, AIRLFu2018} and adversarial imitation learning (AIL) \citep{GAIL2016Ho, GAIL201Finn}. Instead of reducing the problem to supervised learning, these algorithms use reinforcement learning with online interactions to match the long-term visitation distribution of the expert, and agents learn to self-correct when they deviate from the expert support. In particular, AIL formulates this problem as a GAN \citep{Goodfellow2014GenerativeAN}: a discriminator is trained to distinguish between the expert trajectories and those produced by the policy. The policy acts as a generator, producing rollouts from the environment, and is optimized with RL to maximize likelihood produced by the discriminator that the rollout is from the expert distribution. Both model-free \citep{GAIL2016Ho, DAC2019Kostrikov, SAM2019Blonde} and model-based approaches \citep{baram2016modelbased, rafailov2021visual} to the RL optimization problem have been developed. In general, these methods deploy existing online policy optimization algorithms with the discriminator-based reward-learning framework described above.


A crucial difference between the RL setting, for which these online policy optimization algorithms have been developed, and imitation learning is the role of exploration. We argue that RL methods carry out excessive exploration when used in the imitation learning setting, and are thus sub-optimal. As the expert distribution is already empirically known and the goal of the agent is to match that distribution, we argue that this setting is much closer to offline pre-training with online fine-tuning, where the agents learns from offline data and fine-tunes with few online interactions. Moreover, this argument is supported by a recent imitation learning method \citep{garg2021iq}, which draws connections between IL and conservative Q-learning \citep{kumar2020conservative}, an offline RL method.


In this work, we develop a conservative model-based policy optimization algorithm for adversarial imitation learning from pixel observations, which we dub \emph{CMIL}, drawing on ideas from offline RL. We argue that conservatism is inherently well-suited to the imitation learning problem. We justify theoretically that a model-based conservative algorithm is appropriate, and show significant improvements on efficiency and stability over prior imitation learning algorithms. We evaluate our method on a set of challenging environments featuring long horizon manipulation and complex dexterity (fig. \ref{fig:envs}). We further show direct empirical evidence from the environments above for our theoretical bounds and approximations. 

We can summarize our contributions as follows: (i) we re-frame imitation learning as offline pre-training with online fine-tuning; (ii) we augment adversarial imitation learning with conservatism, which improves performance; (iii) we present a theoretical derivation of imitation learning as conservative model-based policy optimization; (iv) we provide a framework for reward-free fine-tuning of world model agents.

\section{Related Work}\label{sec:related_work}
Our work sits at the intersection of imitation learning, high-dimensional model-based RL, and offline RL. We will review relevant works from these fields below. 

\paragraph{World models and RL}
Variational predictive models have demonstrated success in a variety of challenging applications. These works model the high-dimensional observation space as a POMDP and use a recurrent VAE model to jointly learn a compact latent representation and a forward dynamics model. One line of research \citep{hu2022model, hu2021fiery, watter2015embed, zhang2019solar, SLAC2020Lee} utilizes the model for representation purposes only and uses standard RL, control, or imitation on top of it. Others such as \citep{ha2018world, Dreamer2020Hafner, hafner2020mastering, hafner2023dreamerv3} use the learned latent dynamics model to train policies entirely on imagined rollouts. A more recent line of work \citep{hansen2022temporal, hansen2023tdmpc2} learns a dynamics model only using a reward, critic and latent consistency objectives without reconstruction. They then deploy planning-based methods for the actual control part. These line of work was further extended to efficient learning from demonstrations and sparse rewards \citep{hansen2022modem} and deployed on real robot systems \citep{lancaster2023modemv2, feng2023finetuning}.

\paragraph{Conservative MBRL} When learning from prior data model-based algorithms need explicit regularization due to model-hacking issues. Model-based offline RL algorithms~\citep{kidambi2020morel, yu2020mopo, argenson2020model, matsushima2020deployment, swazinna2020overcoming, Rafailov2020LOMPO, yu2021combo} also start by pre-training a dynamics model from prior data. However, many of these approaches \citep{yu2020mopo, yu2021combo, cang2021behavioral} use large off-policy replay buffers, which make them unsuitable for online fine-tuning. Other such as \citep{Kidambi-MOReL-20, matsushima2020deployment} use on-policy optimization, which makes them more suitable for fine-tuning, but also does not allow them to make efficient use of prior data for policy optimization. MoDem \citep{hansen2022modem} uses the core approach of \citep{hansen2022temporal} and makes several modifications to the training pipeline, such as behaviour-cloning pre-training, seed rollouts and prioritized replay to make the fine-tuning approach more efficient. These ideas were further developed in \citep{lancaster2023modemv2, feng2023finetuning} which also consider epistemic uncertainty through a critic ensemble to combat model hacking. Alternatively \citep{rafailov2023moto} builds on the DreamerV2 architecture for the online-fine-tuning. It combined online-model based rollouts with Q learning, policy regularization and ensemble-based epistemic uncertainty. Concurrent work by \citep{gokul2024} also investigate curtailing excessive exploration for inverse reinforcement learning, but focus instead on low-dimensional tasks and use expert resets. 

\paragraph{Model-based imitation learning} Many works have tried combining the benefits of imitation learning and the stable, efficient and predictable nature of model-based learning or planning. A line of works have scaled the approach from \citep{PlanNet2019Hafner} to realistic driving scenarios \citep{hu2021fiery}. In particular \citep{hu2022model} combined model-based representation learning and pure behaviour cloning to achieve state of the art results on the CARLA challenge \citep{Dosovitskiy17}. Other works \citep{baram2016modelbased} use model-based optimization with adversarial imitation learning \citep{GAIL2016Ho}. This approach was also scaled to realistic driving applications in \citep{bronstein2022hierarchical}, which train fully offline, but adds a behaviour cloning regularization term to the policy training. Other works \citep{yue2023clare, demoss2023ditto, zhang2023discriminatorguided} use techniques from inverse RL and offline RL to train fully offline imitation algorithms with world models. Similar to our setting \citep{yin2022planning} uses online model learning in combination with behaviour cloning and planning with a learned discriminator as a cost function and achieves significant improvement in learning efficiency over prior work \citep{rafailov2021visual}.

\section{Preliminaries}\label{section:preliminaries}
In this section we introduce the preliminaries of of variational dynamics models and adversarial imitation learning.

\paragraph{POMDP} We consider the problem setting of learning in partially observed Markov decision processes (POMDPs), which can be described with the tuple: $\mdp = (\mathcal{S}, \mathcal{A}, \mathcal{X}, \mathcal{R}, \dynamics, \mathcal{U}, \gamma)$, where $\state \in \mathcal{S}$ is the state space, $\action \in\mathcal{A}$ is the action space, $\obs \in\mathcal{X}$ is the observation space and $r = \mathcal{R}(\state,\action)$ is a reward function. The state transition dynamics is Markovian and given by the distribution $\state' \sim \dynamics(\cdot|\state,\action)$. Finally, the observations are generated through the observation model $\obs \sim \mathcal{U}(\cdot | \state)$. We do not have access to the underlying dynamics, the true state representation of the POMDP, or the reward function. Following the standard inverse RL framework, the agent is provided with a fixed set of expert demonstrations, which we assume are optimal under the unknown reward function. The agent can interact with the environment and learns a policy $\policy(\action_t | \obs_{\leq t})$ that mimics the expert.

\paragraph{Variational Dynamics Models}
A common approach in model-based learning for POMDPs is to use a variational recurrent state-space model (RSSM) jointly model the belief space and the (latent) transition dynamics as $\dynamics_\theta$ \citep{PlanNet2019Hafner, Dreamer2020Hafner}. The RSSM is optimized via the standard ELBO bound:
\begin{align}\label{eq:model_eq}
\mathbb{E}_{q_{\theta}}\Big[\sum_{t=1}^{\tau} \underbrace{ \log p_{\theta}(\obs_t|\state_t)}_{\text{reconstruction}} - \underbrace{\mathbb{D}_{KL}(q_{\theta}(\state_{t}|\obs_{t}, \state_{t-1}, \action_{t-1})||\dynamics_{\theta}(\state_{t}|\state_{t-1}, \action_{t-1}))}_{\text{latent forward model}}\Big]
\end{align}
where both the inference distribution $q_{\theta}$ and the latent dynamics model $\dynamics_{\theta}$ are Gaussian distributions parameterized by neural networks. Following \citep{rafailov2020offline}, we train a latent ensemble of dynamics models $\{\dynamics_{\theta^i}\}_{i=1}^K$ by selecting a different member of the ensemble to evaluate the above loss at every time step in the trajectory.

\paragraph{Imitation learning as divergence minimization}\label{sec:prelim_divergence_min} In line with prior work, we interpret imitation learning as a divergence minimization problem \citep{GAIL2016Ho, DIV2019Sayed, Ke2019ImitationLA}. We begin by analyzing the MDP case. Let
$
\rho^\policy_\mdp(\state, \action) = (1{-}\gamma) \sum_{t=0}^{\infty} \gamma^t \mathbb{P}^\policy_\mdp(\state_t{=}\state, \action_t{=}\action)
$
be the discounted state-action visitation distribution of a policy $\policy$ in MDP $\mdp$. We can notice by marginalizing over all possible states and actions that this is indeed a valid probability distribution. Then, a divergence minimization objective for imitation learning corresponds to
\begin{equation}
    \label{eq:density_objective}
    \min_\policy \ \ \mathbb{D}(\rho^\policy_\mdp, \rho^E_\mdp),
\end{equation}
where $\rho^E_\mdp$ is the discounted visitation distribution of the expert policy $\policy^E$, and $\mathbb{D}$ is a divergence measure between probability distributions. Essentially, we would like to minimize the divergence between the distributions of expert and policy-induced trajectories. To see why this is a reasonable objective, consider the following proposition:

\begin{proposition}\label{prop:TVbound}
Let $V^{\policy}_{\mdp}$ denote the expected return of a policy $\policy$ in $\mdp$. We we can then bound the sub-optimality of any policy $\policy$ as:
$$
\left|V^{\policy}_{\mdp} - V^{\policy'}_{\mdp'}\right| \leq \frac{2R_{\max}}{1-\gamma} \ \mathbb{D}_{TV}(\rho^\policy_\mdp, \rho^{\policy'}_{\mdp'})
$$
where $R_{\max}$ is the maximum reward in the underlying MDP and $\mathbb{D}_{TV}$ is total variation distance.
\end{proposition}
By optimizing the objective in Eq. \ref{eq:density_objective}, we directly upper bound the policy under-performance. 

\paragraph{ Adversarial Imitation Learning}

With the divergence minimization viewpoint, we follow the reasoning in \citep{GAIL2016Ho} to optimize the bound in Eq. \ref{eq:density_objective} using the following max-min objective.
\begin{align}
    \label{eq:gail_objective}
    \max_\policy \ \min_{D_\psi} \ & \mathbb{E}_{(\state, \action) \sim \rho_\mdp^E} \left[ - \log D_\psi(\state, \action) \right]  + \mathbb{E}_{(\state, \action) \sim \rho_\mdp^\policy} \left[ - \log \left( 1 - D_\psi(\state, \action) \right) \right]
\end{align}
where $D_\psi$ is a classifier used to distinguish between samples from the expert distribution and the policy generated distribution, also referred to as a ``discriminator''. Results from \citet{Goodfellow2014GenerativeAN} and \citet{GAIL2016Ho} suggest that the learning objective in Eq.~\ref{eq:gail_objective} corresponds to the divergence minimization objective in Eq.~\ref{eq:density_objective} with Jensen-Shannon divergence. This is a valid optimization objective, since the total variation distance is upper bounded by the Jensen-Shanon divergence through Pinsker's inequality. A challenge of this objective is that the second term requires on-policy samples, which is inefficient and hard to scale to high dimensions. Off-policy algorithms have been proposed, but can no longer guarantee that the induced visitation distribution of the learned policy will match that of the expert. 

As an alternative, model-based algorithms can utilize a large number of \emph{imagined} rollouts using the learned dynamics model, with periodic model correction. In addition, learning the dynamics model serves as a rich auxiliary task for state representation learning, making policy learning easier and more sample efficient. Below we justify the model-based approach.

\section{Conservative Model-Based Adversarial Imitation Learning}
\label{sec:algo_mdp}
In this section we will justify our choice of model-based RL algorithms and our main theoretical results. Model-based algorithms for RL and IL involve learning an approximate dynamics model $\widehat{\dynamics}$ using environment interactions. The learned dynamics model can be used to construct an approximate MDP $\mdphat$. In our context of imitation learning, learning a dynamics model allows us to generate on-policy samples from $\mdphat$ as a surrogate for samples from $\mdp$, leading to the objective:
\begin{equation}
    \label{eq:model_based_density_objective}
    \min_\policy \ \ \mathbb{D}(\rho^\policy_{\mdphat}, \rho^E_\mdp),
\end{equation}
which can serve as a good proxy to Eq.~\ref{eq:density_objective}. In other words, we can train the imitation learning algorithm on imagined rollouts inside a learned model, which is trained in an off-policy manner. However, this comes at the cost of potential under-performance, due to model mismatch. This intuition can be captured using the following Theorem. 

\begin{theorem}\label{thm:main_theorem} Let $R_{\max}$ be the maximum of the unknown reward in the MDP with unknown dynamics $\dynamics$. For any policy $\policy$, we can bound the sub-optimality with respect to the expert policy $\policy^E$ as:

\begin{align}\label{eq:TVbound}
    \left|{V^{\policy^E}_{\mdp} - V^{\policy}_{\mdp}}\right| \leq & \frac{2R_{\max}}{1-\gamma} \underbrace{{\mathbb{D}_{TV}(\rho^\policy_{{\mdphat}}, \rho^E_\mdp)}}_{\text{distribution mathcing}}  +  \frac{\gamma \cdot R_{\max}}{(1-\gamma)^2}\underbrace{{\mathbb{E}_{\rho^\pi_{\mdphat}}\left[\mathbb{D}_{TV}(\dynamics(s'|s, a), \widehat{\dynamics}(s'|s, a)\right]}}_{\text{model mismatch}}
\end{align}
\end{theorem}

\begin{proof} The proof combines several techniques from prior works on imitation learning and offline model-based RL. We begin with the left-hand side

\begin{equation*}
    \left|V^{\policy^E}_{\mdp} - V^{\policy}_{\mdp}\right| \leq \underbrace{\left|V^{\policy^E}_{\mdp} - V^{\policy}_{\mdphat}\right|}_{\text{Term I}} + \underbrace{\left|V^{\policy}_{\mdphat} - V^{\policy}_{\mdp}\right|}_{\text{Term II}}
\end{equation*}
which is a simple application of the triangle inequality. Through Proposition \ref{prop:TVbound} the first term above produces the distribution matching component of the objective in Eq. \ref{eq:TVbound}. The second term yields the model mismatch part of the objective via the "simulation lemma" \citep{Abbeel2005ExplorationAA}:

\begin{align*}
 &\left | V^{\policy}_{\mdphat}(\state_0) - V^{\policy}_{\mdp}(\state_0)\right | = \left|\mathbb{E}_{\policy, \mdphat}\left[r(\state_0, \action_0) + \gamma V^{\policy}_{\mdphat}(\state_1)\right] - \mathbb{E}_{\policy, \mdp}\left[r(\state_0, \action_0) + \gamma V^{\policy}_{\mdp}(\state_1)\right]\right| \\
 &= \gamma\left|\mathbb{E}_{\policy, \mdphat}\left[V^{\policy}_{\mdphat}(\state_1)\right] - \mathbb{E}_{\policy, \mdp}\left[V^{\policy}_{\mdp}(\state_1)\right]\right| \\
 &\leq \gamma\left|\mathbb{E}_{\policy, \mdp}\left[V^{\policy}_{\mdp}(\state_1)\right] - \mathbb{E}_{\policy, \mdphat}\left[V^{\policy}_{\mdp}(\state_1)\right]\right| +
\gamma\left|\mathbb{E}_{\policy, \mdphat}\left[V^{\policy}_{\mdp}(\state_1)\right] - \mathbb{E}_{\policy, \mdphat}\left[V^{\policy}_{\mdphat}(\state_1)\right]\right| \\
&\leq \gamma \frac{R_{\max}}{1-\gamma}\mathbb{E}_{\policy}\left[\mathbb{D}_{TV}(\dynamics(\state_1|\state_0, \action_0), \widehat{\dynamics}(\state_1|\state_0, \action_0))\right] + \gamma\mathbb{E}_{\policy, \mdphat}\left[\left|V^{\policy}_{\mdp}(\state_1)-V^{\policy}_{\mdphat}(\state_1)\right|\right]
\end{align*}

The first equation follows from the definition of the value function, the next inequality is a direct application of the triangle inequality. The final inequality follows from Proposition \ref{prop:TVbound}. We can then recursively apply the same reasoning to the final term to obtain the bound:

\begin{align*}
\left|V^{\policy}_{\mdphat}(\state_0) - V^{\policy}_{\mdp}(\state_0)\right| \leq & \frac{R_{\max}}{1-\gamma}\mathbb{E}_{\policy, \mdphat}\left[ \sum_{t=0}^{\infty}\gamma^{t+1}\mathbb{D}_{TV}(\dynamics(\state_{t+1}|\state_t, \action_t), \widehat{\dynamics}(\state_{t+1}|\state_t, \action_t))\right] = \\ & \frac{\gamma \cdot R_{\max}}{(1-\gamma)^2}\mathbb{E}_{\rho^\pi_{\mdphat}}\left[\mathbb{D}_{TV}(\dynamics(s'|s, a), \widehat{\dynamics}(s'|s, a)\right]
\end{align*}

\end{proof}
That is, if we want to use imagined rollouts for on-policy training of the distribution matching objective, as outlined in Section \ref{sec:prelim_divergence_min}, we need to pay the model mismatch cost in potential under-performance. 
Prior works \citep{Kidambi-MOReL-20, MOPO, rafailov2020offline} consider the model mismatch objective in the context of fully offline model-based reinforcement learning. Usually such approaches are not applied in online scenarios, or can even be flipped and used as exploration algorithms \citep{sekar2020planning} as they explicitly encourage the agent to remain within parts of the state space with high confidence and hinder exploration. However, in the imitation learning case, we can treat the problem as online fine-tuning, since we have explicit access to data from the expert. This makes the problem similar to the fine-tuning setting in that it does not require extensive exploration, where conservative MBRL has show meaningful improvement recently \citep{rafailov2023moto, feng2023finetuning}.

\subsection{Extension to POMDPs}
The results of Section \ref{sec:algo_mdp} were derived under the MDP formulation. In this section we will formulate how they can be translated into the more general POMDP setting, which we explore in this work. Consider the upper bound under Theorem \ref{thm:main_theorem}. Theorem 1 in  \citep{rafailov2021visual} justifies using the belief representation for bounding distribution matching objective in Eq. \ref{eq:TVbound}. Here we will extend this interpretation to justify using latent transition models for the model mismatch objective as well. We formulate this in the below Theorem:

\begin{theorem}
Consider two POMDPs $\mdp$ and $\mdp'$ with identical components, besides potentially different transition models $\dynamics$ and $\dynamics'$. We denote the density over observations at time $t\!+\!1$ as $P_{\mdp}({}\cdot{} |x_{\leq t}, a_{\leq t})$. The, the following result holds

\begin{equation}\label{eq:pomdp_model_mismatch}
\mathbb{D}_f(P_{\mdp}(x_{t+1}|x_{\leq t}, a_{\leq t}), P_{\mdp'}(x_{t+1}|x_{\leq t}, a_{\leq t})) \leq \mathbb{D}_f(\dynamics(s'|s, a), \dynamics'(s'|s, a))
\end{equation}
where $\mathbb{D}_f$ is an arbitrary $f$-divergence.
\begin{proof}
From the data processing inequality \cite{cover1999elements}, and the definition of a POMDP, it follows that
\begin{equation}
    \mathbb{D}_f(P_{\mdp}(x_{t+1}|x_{\leq t}, a_{\leq t}), P_{\mdp'}(x_{t+1}|x_{\leq t}, a_{\leq t})) \leq \mathbb{D}_f(P_{\mdp}(s_{t+1}|x_{\leq t}, a_{\leq t}), P_{\mdp'}(s_{t+1}|x_{\leq t}, a_{\leq t}))
\end{equation}
However, we also have that $P_{\mdp}(s_{t+1}|x_{\leq t}, a_{\leq t}) = \dynamics(s_{t+1}|s_t, a_t)$ and $P_{\mdp'}(s_{t+1}|x_{\leq t}, a_{\leq t}) = \dynamics'(s_{t+1}|s_t, a_t)$ which gives us the final result.
\end{proof}

\end{theorem}
This result, allows us to justify latent model uncertainty in as an upper bound on the 

            

\section{Our Method}
Our full algorithm has several components: (i)  variational dynamics model (as described in sec. \ref{section:preliminaries}), (ii)  state-action discriminator, and (iii) the actor-critic policy optimization. We will discuss all of these in more detail. We train the model, discriminator, actor and critic simultaneously.

\paragraph{Adversarial Formulation} Directly applying Theorem 1 in  \citep{rafailov2021visual}, we can bound the objective of Eq. \ref{eq:TVbound} by optimizing the same bound in the learned models' latent belief space. In more detail, we consider sequences of data of the form $\tau = (\obs_{1:T}, \action_{1:T})$. At each agent training step, we infer latent states $\state_{1:T}^0\sim q_{\theta}(\state_{1:T}|\obs_{1:T}, \action_{1:T})$. We also denote $\action_{1:T}$ as $\action_{1:T}^0$. Using these states as starting points, we use the policy $\policy_{\psi}$ to generate $H$-step rollouts steps with the following notation: $\hat{\action}_j^t\sim \pi_{\psi}(\action|\hat{\state}_j^t)$ and  $\hat{\state}_{j}^{t+1} \sim p_{\theta}(\state|\hat{\action}_j^t, \hat{\state}_j^t)$. Following standard off-policy learning algorithms, we use critics $\{Q_{\psi^i}\}_{i=1}^2$ and and target networks $\{\widebar{Q}_{\psi^i}\}_{i=1}^2$.

We can bound the distribution matching component of Eq. \ref{eq:TVbound} through Pinsker's inequality and follow a standard adversarial approach by training a discriminator \citep{DIV2019Sayed}:
\begin{equation}
    \label{eq:our_objective_mdp}
    \min_{D_\psi} - \frac{1}{T}\mathbb{E}_{\tau\sim\mathcal{D}^E, q_{\theta}}\left[\sum_{i=1}^T\log  D_\psi(\state_i^E, \action_i^E)\right] -\frac{1}{HT}\mathbb{E}_{\dynamics_{\theta}, \pi_{\psi}}\left[\sum_{i=1, t=0}^{T, H-1}\log (1-D_\psi(\state_i^t, \action_i^t)) \right]
\end{equation}
where $\state_{1:T}^E\sim q_{\theta}(\state_{1:T}|\obs_{1:T}^T, \action_{1:T}^E)$ are the expert's inferred latent states and actions.  On-policy samples from the model also give theoretical justification for this discriminator learning objective as the expectation is taken under the current policy. We cannot directly optimize the model mismatch component of Eq. \ref{eq:TVbound}, as we cannot directly estimate divergence factor. Instead, following prior work \citep{MOPO, Kidambi-MOReL-20, rafailov2020offline}, we optimize a surrogate objective based on ensemble model disagreement. In particular, we use
\begin{equation}
    \label{eq:model_error_estimator}
\mathbb{E}_{\rho^\pi_{\mdphat}}\left[\mathbb{D}_{TV}(\dynamics(s'|s, a), \widehat{\dynamics}(s'|s, a)\right]\approx \mathbb{E}_{\pi, \widehat{\dynamics}}[\text{std}(\{\mu_{\theta^i}(\state,\action)\}_{i=1}^K)]
\end{equation}
where $\mu_{\theta^i}(\state,\action)$ is the parameterization of the mean of $i$-th Gaussian latent model ${\dynamics}_{\theta^i}(\cdot|\state, \action)$. The combined final reward for the agent is then:
\begin{equation}\label{eq:learned_reward}
    \rew(\state, \action) = \log D_{\psi}(\state, \action)-\log(1-D_{\psi}(\state, \action)) - \alpha\text{std}(\{\mu_{\theta^i}(\state,\action)\}_{i=1}^K)
\end{equation}

where $\alpha$ is a tunable hyper-parameter. We note that this still a fully differentiable function of the state. We can then label all data and model-sampled latent transitions $\rew_i^t = \rew(\state_i^t, \action_i^t)$. However, we use the uncertainty regularization only on the model-generated rollouts and not on the replay buffer data.


\paragraph{Actor Optimization} Once we have rewards, we can calculate Monte-Carlo based policy returns:
\begin{align*} 
& V_0^{\pi_{\psi}}(\hat{\state}_j^t) = \min_{i=1,2}\{Q_{\psi^i}(\hat{\state}_j^t, \hat{\action}_j^t)\} \nonumber \\
& V_K^{\pi_{\psi}}(\hat{\state}_j^t) = \sum_{k=1}^K \gamma^{k-1}\hat{\rew}_j^{k+t} + \gamma ^{K}V_0^{\pi_{\psi}}(\hat{\state}_j^{t+K}) \nonumber
\end{align*}
And compute the standard TD$(\lambda)$ estimate:
\begin{equation} \label{eq:MC_value}
 V^{\pi_{\psi}}(\hat{\state}_j^t) = (1-\lambda)\sum_{k = 1}^{H-t-1}\lambda^{k-1} V_k^{\pi_{\psi}}(\hat{\state}_j^t) + \lambda^{H-t-1}V_{H-t}^{\pi_{\psi}}(\hat{\state}_j^t)
\end{equation}

We can then train the actor to maximize the  TD$(\lambda)$ estimate as:
\begin{equation}\label{eq:actor_objective}
    \mathcal{L}^{\text{model}}_{\pi_{\psi}} = -\frac{1}{H T}\mathbb{E}_{\dynamics_{\theta}, \pi_{\psi}}\Bigg[\sum_{t=0, j=1}^{H-1, T}\lambda V^{\pi_{\psi}}(\hat{\state}_j^t) + (1-\lambda)V_0^{\pi_{\psi}}(\hat{\state}_j^t)\Bigg].
\end{equation}
Note that this maximizes the expected MC return at both the dataset and rollout states. Moreover, this fully differentiable function of the policy parameters, as we can differentiate through the model and reward function. 

\noindent\textbf{Critic Optimization} We can use MC return estimates to train the critics via a Bellman backup objective. We recompute the critic target values $\widebar{V}^k(\hat{\state}_j^t)$ for all states similarly to Eq. \ref{eq:MC_value} using the target networks $\{\widebar{Q}_{\psi^i}\}_{i=1}^2$. Critics are trained on both the model-generated and real data with two losses:
\begin{equation}\label{model_critic}
    \mathcal{L}^{\text{model}}_{Q_{\psi^i}} = \frac{1}{HT}\mathbb{E}_{\dynamics_{\theta}, \pi_{\psi}}\Bigg[\sum_{t=0, j=1}^{H-1, T}(Q_{\psi^i}(\hat{\state}_j^t, \hat{\action}_j^t) - \widebar{V}^{\pi_{\psi}}(\hat{\state}_j^t))^2 \Bigg]
\end{equation}
\begin{equation}\label{data_critic}
    \mathcal{L}^{\text{data}}_{Q_{\psi^i}} = \frac{1}{T-1}\Bigg[\sum_{j=1}^{T-1}\Big(Q_{\psi^i}(\state_j^0, \action_j^0) -
    (\hat{\rew}_{j+1}^0 + \gamma \big((1-\lambda) \widebar{V}_0^{\pi_{\psi}}(\state_{j+1}^0) + \lambda \widebar{V}^{\pi_{\psi}}(\state_{j+1}^0)\big)\Big)^2 \Bigg]
\end{equation}
The final critic loss is a combination of the two losses:
\begin{equation}\label{final_critic}
    \mathcal{L}_{Q_{\psi^i}} = \mathcal{L}^{\text{model}}_{Q_{\psi^i}} + \mathcal{L}^{\text{data}}_{Q_{\psi^i}}
\end{equation}
The loss $\mathcal{L}^{\text{data}}_{Q_{\psi^i}}$ is computed on transitions sampled form the dataset trajectories through the inference model $q_{\theta}$. Note that the real rollouts do not receive the uncertainty penalty. Training the critics on the available expert demonstrations serves as a strong supervision, and aids in minimizing the shift from the expect-induced state-action distribution. 


\paragraph{Practical Implementation Notes}
We share some notable implementation\footnote{\url{https://www.github.com/victorkolev/cmil}} details that we have employed. We build on the DreamerV2 architecture, keeping all hyperparameters default. Beyond the optimization procedure outlined above, we add a few details to help with training. As in \citep{hansen2022modem}, we initialize the policy with behavior cloning, and then roll out ``seed'' episodes to add to the replay buffer. Additionally, we add behaviour cloning loss with a hyperparameter $\beta$ to the actor loss to regularize actor training \citep{rafailov2023moto}. Crucially, we add Gaussian noise $\mathcal{N}(0, \sigma^2)$ to the state-action input vector in training the discriminator (eq. 3), which serves to regularize training. We found this to be important in training stable and robust discriminator, and hence an integral part of the algorithm's success. We apply the same regularization to the expert and the policy. While we have experimented with different discriminator training regimes (e.g. using Wasserstein distance),  they did not yield a significant improvement in results.

\section{Experiments \& results}

\begin{figure}
    \centering
    \includegraphics[width=0.6\linewidth]{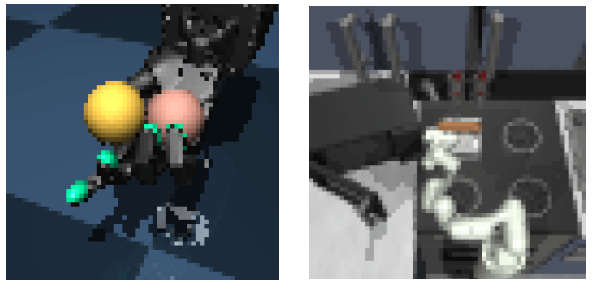}
    \caption{The suite of environments. \textbf{Left}: ShadowHand Baoding Balls.  \textbf{Right}: Franka Kitchen.  }
    \label{fig:envs}
\end{figure}

We devise a diverse experimental setup to evaluate our approach, \emph{CMIL}, on the relevant domains -- dexterous manipulation (ShadowHand Baoding balls) and long-horizon planning (Franka Kitchen). The first environment consists of manipulating a 20-DoF 24-joint Shadow Hand robot that has to revolve two Baoding balls, and in the latter environment the objective of the task is to use a 7-DoF Franka Emica Panda robot to manipulate four objects in the kitchen in a certain order (open the microwave, move the kettle, turn on the light, open a cabinet); this happens over $\approx 250$ time steps, and hence demands long-horizon planning, in addition to learning dexterity. We focus on the high-dimensional learning setting and use only pixel observations for both environments. Note that we have access to action information of the expert demonstrations, i.e. we are targeting the setting where we can extract actions, but it is non-trivial to define a complete state and form an MDP. We use 10-20 expert demonstrations for each task, which are collected from expert policies\footnote{Demonstration sets are published along with code}. We use the reward functions only as oracles for evaluation, but not at any stage during training. 

In the experimental results, we will be looking at performance measured by success rate and sample efficiency of learning. We choose success rate as a metric instead of total reward to be faithful to the reward-free setting, and also because it is more representative of the agent’s success in learning the expert distribution. We run three random seeds and present training curves with 95\% bootstrap confidence intervals. 


\paragraph{Baselines}
We compare CMIL with three baselines: (i) standard behavior cloning, (ii) P-DAC, an AIL model-free algorithm designed to work from pixel observations, and (iii) VMAIL, a model-based imitation learning algorithm, which does not use conservatism. We also attempt to benchmark with IQL \citep{garg2021iq}, but could not stabilize the algorithm for continuous control from pixel observations. With this set of baselines, we cover both standard approaches (BC),  and leading model-free and model-based imitation learning algorithms. We use default hyperparameters for the baselines.

\paragraph{Results: performance \& sample efficiency}
Training curves are presented in figure \ref{fig:results}. We see CMIL out-performing baselines in both environments, notably achieving close to expert-level performance with few environment interactions (250K steps), demonstrating improved sample efficiency to previous methods which are typically trained on the order of 1M steps.

\begin{figure}
    \centering
    \begin{minipage}{0.48\textwidth}
        \centering
        \includegraphics[width=\textwidth]{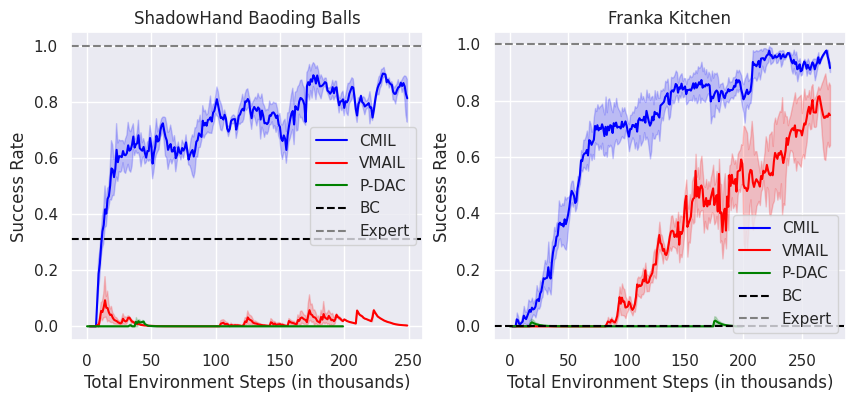} 
        \caption{Training curve of success rate of our approach, CMIL, vs. baselines. }
        \label{fig:results}
    \end{minipage}\hfill
    \begin{minipage}{0.48\textwidth}
        \centering
        \includegraphics[width=\textwidth]{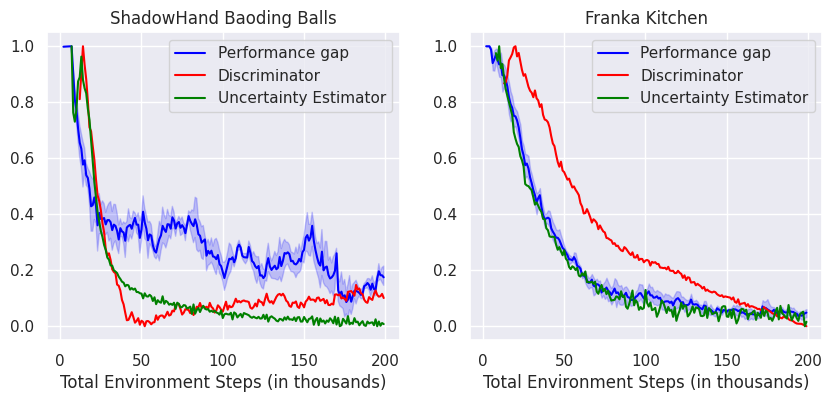} 
        \caption{Empirical estimates of elements from performance gap bound.}
        \label{fig:theory}
    \end{minipage}
\end{figure}


We note that our algorithm is the only one to solve the ShadowHand Baoding Balls environment. We postulate this is due to the higher dimensionality of the action space, which demands conservatism to constrain the search space. While baselines also solve the Kitchen environment, we note that there is a tangible difference in both stability and sample efficiency between our algorithm and baselines.  All architectural hyperparameters have been kept at default from the DreamerV2 architecture \citep{Dreamer2020Hafner}, and we only tune the Gaussian noise added to the discriminator in training ( $\sigma^2{=}2.5$ tends to work well across the board), the disagreement penalty $\alpha$ (we keep $\alpha{=}1$0) and the behaviour cloning regularization parameter ($\beta{=}10$). We note the algorithm is quite resilient to the choice of the latter two, but the former is integral to successful discriminator training and hence accurate rewards.

\paragraph{Empirical verification of theoretical results}
We aim to validate empirically the claims of Theorem \ref{thm:main_theorem} and to qualitatively evaluate the discriminator as an approximation of the reward. Theorem \ref{thm:main_theorem} states that the performance gap between the expert and the current policy is upper bounded by the sum of (i) the distribution match between the expert and the policy learned in the model, and (ii) the error of the model. We can approximate the distribution matching component as $D_\psi(\hat\rho^E_\mdp) - D_\psi(\hat\rho^\policy_{\hat\mdp})$, where $\hat\rho^\policy$ is the empirical state-action distribution. On the other hand, we approximate the model mismatch component via an admissible error estimator as in Eq. \ref{eq:model_error_estimator}, which we also calculate. We plot those two quantities (normalized min=0, max=1), as well as the oracle reward performance gap, in fig. \ref{fig:theory}. 

Noting the plots for the Franka Kitchen, we see the discriminator is a strong upper bound to the performance gap, while providing a reasonably accurate (dense) reward approximation. On the other hand, in the ShadowHand environment, we see the discriminator diverging initially and then converging back to the performance gap. In both environments, the uncertainty estimator has a similar behavior, indicating that the model mismatch is stably decreasing.

\section{Conclusion}
In this work we argue that policy optimization algorithms designed for online RL are not well suited to the IRL/AIL setting as they carry out excessive exploration and induce additional distributional shifts. We focus on the model-based case, and argue that conservative models used for offline RL are better suited to the AIL setting. We pose imitation learning as a fine-tuning problem, rather than a purely RL one, and we draw theoretical connections to offline RL and conservative fine-tuning algorithms. We provide a conservative optimization bound, as well as a practical algorithm and evaluate it on challenging manipulation tasks. The proposed algorithm achieves faster and more stable performance as compared to previous imitation learning approaches. In future work we plan to evaluate our method on further domains.

\acks{Chelsea Finn is a CIFAR Fellow in the Learning in Machines and Brains program. This work was also supported by ONR grant N00014-22-1-2621 and the Volkswagen Group.}

\bibliography{bibliography}

\end{document}